\documentclass{article}
\usepackage{ijcai18}

\usepackage{times}
\usepackage{xcolor}
\usepackage{soul}
\usepackage[utf8]{inputenc}
\usepackage[small]{caption}
\usepackage{epsfig}

\usepackage{amssymb}
\usepackage{amsthm}

\usepackage{times}
\usepackage{helvet}
\usepackage{courier}
\usepackage{mathrsfs}
\usepackage{amsfonts,amssymb}
\usepackage{amsmath}
\usepackage{CJK}
\usepackage{cases}
\usepackage{bm}
\usepackage{float}
\usepackage{multicol}
\usepackage{graphicx}
\usepackage{subfigure}
\usepackage{amsmath}
\newtheorem{theorem}{Theorem}
\newtheorem{assumption}{Assumption}
\newtheorem{lemma}{Lemma}
\newtheorem{myDef}{Definition}
\newtheorem{remark}{Remark}
\usepackage{caption}
\usepackage{subfigure}
\usepackage{graphicx}

\title{A Unified Approach for Multi-step Temporal-Difference Learning with Eligibility Traces in Reinforcement Learning}

\author{ID:3116 
Laboratoire d'Analyse et Modélisation des Systèmes pour l'Aide à la Décision (LAMSADE)  \\
pcchair@ijcai-18.org}

\author{
Long Yang, 
Minhao Shi, 
Qian Zheng, 
Wenjia Meng,
Gang Pan
\\ 
Zhejiang University \\
\{yanglong,minhaowill,qianzheng,mengwenjia,gpan\}@zju.edu.cn
}

\begin{document}

\maketitle

\begin{abstract}
Recently, a new multi-step temporal learning algorithm, called $Q(\sigma)$, unifies $n$-step Tree-Backup (when $\sigma=0$) and $n$-step Sarsa (when $\sigma=1$) by introducing a sampling parameter $\sigma$. However, similar to other multi-step temporal-difference learning algorithms, $Q(\sigma)$ needs much memory consumption and computation time. 
Eligibility trace is an important mechanism to transform the off-line updates into efficient on-line ones which consume less memory and computation time.
In this paper, we further develop the original $Q(\sigma)$, combine it with eligibility traces and propose a new algorithm, called $Q(\sigma ,\lambda)$, in which $\lambda$ is trace-decay parameter. This idea unifies Sarsa$(\lambda)$ (when $\sigma =1$) and $Q^{\pi}(\lambda)$  (when $\sigma =0$). Furthermore, we give an upper error bound of $Q(\sigma ,\lambda)$ \emph{policy evaluation} algorithm. We prove that $Q(\sigma,\lambda)$ \emph{control} algorithm can converge to the optimal value function exponentially.
We also empirically compare it with conventional temporal-difference learning methods. Results show that, with an intermediate value of $\sigma$, $Q(\sigma ,\lambda)$ creates a mixture of the existing algorithms that can learn the optimal value significantly faster than the extreme end ($\sigma=0$, or $1$).
\end{abstract}

\section{Introduction}
In reinforcement learning, experiences are sequences of states, actions and rewards that generated by the agent interacts with environment. The agent's goal is learning from experiences and seeking an optimal policy from the delayed reward decision system. There are two fundamental mechanisms have been studied, one is temporal-difference (TD) learning method which is a combination of Monte Carlo method and dynamic programming ~\cite{sutton1988learning}. The other one is eligibility trace~\cite{sutton1984temporal,watkins1989learning}, which is a short-term memory process as a function of states. TD learning combining with eligibility trace provides a bridge between one-step learning and Monte Carlo methods through the trace-decay parameter $\lambda$ ~\cite{sutton1988learning}.

Recently, Multi-step $Q(\sigma)$~\cite{sutton2017reinforcement} unifies $n$-step Sarsa ($\sigma=1$, \emph{full-sampling}) and $n$-step Tree-backup ($\sigma=0$, \emph{pure-expectation}). For some intermediate value $\sigma (0<\sigma<1)$, $Q(\sigma)$ creates a mixture of full-sampling and pure-expectation approach, can perform better than the extreme case $\sigma=0$ or $1$~\cite{de2018multi}.

The results in ~\cite{de2018multi}  implies a fundamental trade-off problem in reinforcement learning : \emph{should one estimates the value function by adopting pure-expectation ($\sigma=0$) algorithm or full-sampling ($\sigma=1$) algorithm}? Although pure-expectation approach has lower variance, it needs more complex and larger calculation ~\cite{van2009theoretical}. On the other hand, full-sampling algorithm needs smaller calculation time, however, it may have a worse asymptotic performance~\cite{de2018multi}. Multi-step $Q(\sigma)$ ~\cite{sutton2017reinforcement} firstly attempts to combine pure-expectation with full-sample algorithms, however, multi-step temporal-difference learning is too expensive during the training. In this paper, we try to combine the $Q(\sigma)$ algorithm with eligibility trace, and create a new algorithm, called $Q(\sigma,\lambda)$. Our $Q(\sigma,\lambda)$ unifies the Sarsa$(\lambda)$ algorithm ~\cite{rummery1994line}
and $Q^{\pi}(\lambda)$  algorithm ~\cite{H2016}.  When $\sigma$ varies from 0 to 1, $Q(\sigma,\lambda)$ changes continuously from Sarsa$(\lambda)$ ($\sigma=1$ in $Q(\sigma,\lambda)$) to $Q^{\pi}(\lambda)$ ($\sigma=0$ in $Q(\sigma,\lambda)$). In this paper, we also focus on the trade-off between pure-expectation and full-sample in \emph{control task}, our experiments show that an intermediate value $\sigma$ can achieve a better performance than extreme case.

Our contributions are summaried as follows:
\begin{itemize}
\item We define a new operator \emph{mixed-sampling operator}  through which we can deduce the corresponding policy evaluation algorithm and control algorithm .
\item For new policy evaluation algorithm, we give its upper error bound.
\item We present an new algorithm $Q(\sigma,\lambda)$ which unifies Sarsa$(\lambda)$ and $Q^{\pi}(\lambda)$. For the control problem, we prove that both of the off-line and on-line $Q(\sigma,\lambda)$ algorithm can converge to the optimal value function.
\end{itemize}

\section{Framework and Notation}
The standard episodic reinforcement learning framework ~\cite{sutton2017reinforcement} is often formalized as \emph{Markov decision processes} (MDPs). Such framework considers 5-tuples form $\mathcal{M}=(\mathcal{S},\mathcal{A},\mathcal{P},\mathcal{R},\gamma)$, where $\mathcal{S}$ indicates the set of all states,
$\mathcal{A}$ indicates the set of all actions,
$P_{s s^{'}}^a$ indicates a state-transition probability from state $s$ to state $s^{'}$ under taking action $a$, $a\in\mathcal{A},s^{'},s\in\mathcal{S}$; $R_{ss^{'}}^a$ indicates the expected reward for a transition,
$\gamma$ is the discount factor. In this paper, we denote $\{(S_{t},A_{t},R_{t})\}_{t\ge0}$ as a \emph{trajectory} of the state-reward sequence in one episode.A \emph{policy} $\pi$ is a probability distribution on $\mathcal{S}\times\mathcal{A}$ and \emph{stationary policy} is a policy that does not change over time.

Consider the $\emph{state-action value} $ $q$ maps on $\mathcal{S}\times\mathcal{A}$ to $\mathbb{R}$, for a given policy $\pi$, has a corresponding state-action value:
$$q^{\pi}(s,a) = \mathbb{E}_{\pi}[\sum_{t=0}^{\infty}\gamma^{t}R_{t}|S_{0} = s,A_{0}=a].$$
\emph{Optimal state-action value} is defined as:
$$q^{*}(s,a) = \max_{\pi}q^{\pi}(s,a).$$
$\emph{Bellman operator}$ $\mathcal{T}^{\pi}$
\begin{flalign}
\mathcal{T}^{\pi} q&=\mathcal{R}^{\pi}+\gamma \mathcal{P}^{\pi}q,
 \end{flalign}
$\emph{Bellman optimality operator}$  $\mathcal{T}^{*}$
\begin{flalign}
\mathcal{T}^{*} q&=\max_{\pi}({\mathcal{R}^{\pi}+\gamma \mathcal{P}^{\pi}q}),
 \end{flalign}
where $\mathcal{R}^{\pi}$ $\in\mathbb{R}^{|\mathcal{S}|\times|\mathcal{A}|}$ and $\mathcal{P}^{\pi}$ $\in\mathbb{R}^{|\mathcal{S}| \times |\mathcal{S}|}$, the corresponding entry is:
$$\mathcal{R}^{\pi}(s,a)=\sum_{s^{'} \in \mathcal{S}}P^{a}_{ss^{'}}R_{ss^{'}}^{a},\mathcal{P}^{\pi}_{ss^{'}}= \sum_{a \in \mathcal{A}}\pi(s,a)P^{a}_{ss^{'}}.$$
Value function $q^{\pi}$ and $q^{*}$ satisfy the following \emph{Bellman equation} and \emph{optimal  Bellman equation} correspondingly:
$$\mathcal{T}^{\pi}q^{\pi}=q^{\pi},\hspace{0.2cm}\mathcal{T}^{*}q^{*}=q^{*}.$$
Both $\mathcal{T}^{\pi}$ and $\mathcal{T}^{*}$ are $\gamma$-contraction operator in the sup-norm, that is to say, $\|\mathcal{T}Q_{1}-\mathcal{T}Q_{2}\|_{\infty}\leq \gamma\|Q_{1}-Q_{2}\|_{\infty}$ for any $Q_{1},Q_{2}$
 $\in\mathbb{R}^{|\mathcal{S}| \times |\mathcal{A} |}$, $\mathcal{T}=\mathcal{T}^{\pi}$ or $\mathcal{T}^{*}$.  From the fact that fixed point of contraction operator is unique, the \emph{value iteration} converges: $(\mathcal{T}^{\pi})^{n}Q\rightarrow q^{\pi}$, $(\mathcal{T^{*}})^{n}Q\rightarrow q^{*}$, as $n \rightarrow \infty$, for any initial $Q$ ~\cite{bertsekas2005dynamic}.

Unfortunately, both the system (1) and (2) can not be solved directly because of fact that the $\mathcal{P}$ and $\mathcal{R}$ in the environment are usually unknown. A practical model in reinforcement learning has not been available, called, \emph{model free}.

\subsection{One-step TD Learning Algorithms}
TD learning algorithm~\cite{sutton1984temporal,sutton1988learning} is one of the most significant algorithms in model free reinforcement learning, the idea of \emph{bootstrapping} is critical to TD learning: the evluation of the value function are used as targets during the learning process.

Given a target policy $\pi$ which is to be learned and a behavior policy $\mu$ that generates the trajectory $\{(S_{t},A_{t},R_{t})\}_{t\ge 0}$, if $\pi=\mu$, the learning is called \emph{on-policy} learning, otherwise it is \emph{off-policy} learning.

\textbf{Sarsa}: For a given sample transition ($S,A,R,S^{'},A^{'}$), \emph{Sarsa} ~\cite{rummery1994line} is a on-policy learning algorithm and its updates $Q$ value as follows:
\begin{flalign}
Q_{k+1}(S,A)&=Q_{k}(S,A) + \alpha_{k}\delta_{k}^{S},\\
   \delta_{k}^{S}&=R+\gamma Q_{k}(S^{'},A^{'}) - Q_{k}(S,A),
 \end{flalign}
where $\delta_{k}^{S}$ is the \emph{k}-th TD error, $\alpha_{k}$ is \emph{stepsize}.

\textbf{Expected-Sarsa}: \emph{Expected-Sarsa} ~\cite{van2009theoretical} uses expectation of all the next state-action value pairs according to the target policy $\pi$ to  estimate $Q$ value as follows:
\begin{flalign}
\nonumber
Q_{k+1}(S,A)&=Q_{k}(S,A) + \alpha_{k}\delta_{k}^{ES},\\
\delta_{k}^{ES}&=R+\gamma \mathbb{E}_{\pi}[Q_{k}(S^{'},\cdot)] - Q_{k}(S,A),\\
\nonumber
 &=R+\sum_{a\in \mathcal{A}}\pi(S^{'},a)Q_{k}(S^{'},a)-Q_{k}(S,A),
\end{flalign}
where $\delta_{k}^{ES}$ is the \emph{k}-th \emph{expected TD error}. Expected-Sarsa is a off-policy learning algorithm if $\mu\neq\pi$, for example, when $\pi$ is greedy with respect to $Q$ then Expected-Sarsa is restricted to \emph{Q-Learning}~\cite{watkins1989learning}. If the trajectory was generated by $\pi$, Expected-Sarsa is  a on-policy algorithm~\cite{van2009theoretical}.

The above two algorithms are guaranteed convergence under some conditions ~\cite{singh2000convergence,van2009theoretical}.

$\bm{Q}(\bm{\sigma})$ :  One-step $Q(\sigma)$ ~\cite{sutton2017reinforcement,de2018multi} is a weighted average between the Sarsa update and  Expected Sarsa  update through sampling parameter $\sigma$:
 \begin{flalign}
 \nonumber
Q_{k+1}(S,A)&=Q_{k}(S,A) + \alpha_{k}\delta_{k}^{\sigma},\\
   \delta_{k}^{\sigma}&=\sigma\delta_{t}^{S}+(1-\sigma)\delta_{t}^{ES},
 \end{flalign}
 Where $\sigma\in[0,1]$ is \emph{degree of sampling}, $\sigma=1$ denoting \emph{full-sampling} and $\sigma=0$ denoting a \emph{pure-expectation} with no sampling, $\delta_{t}^{S},\delta_{t}^{ES}$ are in (4) and (5).

\subsection{$\lambda$-Return Algorithm}
One-step TD learning algorithm can be generalized to multi-step bootstrapping learning method.  The $\lambda$-\emph{return} algorithm ~\cite{watkins1989learning} is a particular way to mix many multi-step TD learning algorithms through weighting $n$-\emph{step returns} proportionally to $\lambda^{n-1}$.

$\lambda$-\emph{operator}\footnote{\text{The notation is coincident with textbook ~\cite{bertsekas2012dynamic}}.} $\mathcal{T}^{\pi}_{\lambda}$ is a flexible way to express $\lambda$-return algorithm,  consider a trajectory $\{(S_{t},A_{t},R_{t})\}_{t\ge0}$,
\begin{eqnarray*}
(\mathcal{T}^{\pi}_{\lambda}q)(s,a)&=&\Big\{(1-\lambda)\sum_{n=0}^{\infty}\lambda^{n}(\mathcal{T}^{\pi})^{n+1}q\Big\}(s,a)\\
&=&\sum_{n=0}^{\infty}(1-\lambda)\lambda^{n}\mathbb{E}_{\pi}[G_{n}|S_{0}=s,A_{0}=a]\\
&= &q(s,a)+\sum_{n=0}^{\infty}(\lambda\gamma)^{n}\mathbb{E}_{\pi}[\delta_{n}|S_{0}=s,A_{0}=a]
\end{eqnarray*}
where $G_{n}=\sum_{t=0}^{n}\gamma^{t}R_{t}+\gamma^{n}Q(S_{n+1},A_{n+1})$ is $n$-\emph{step returns} from initial state-action pair $(S_{0},A_{0})$, the term $\sum_{n=0}^{\infty}(1-\lambda)\lambda^{n}G_{n}$, 
called $\lambda$-\emph{returns}, and $\delta_{n}=R_{n}+\gamma Q(S_{n+1},A_{n+1})-Q(S_{n},A_{n})$.

Based on the fact that $q^{\pi}$ is fixed point of $\mathcal{T}^{\pi}$, $q^{\pi}$ remains the fixed point of $\mathcal{T}^{\pi}_{\lambda}$.  When $\lambda=0$, $\mathcal{T}^{\pi}_{\lambda}$ is equal to the usual Bellman operator
$\mathcal{T}^{\pi}$. When $\lambda=1$ , the evaluation of $Q_{k+1}=\mathcal{T}^{\pi}_{\lambda}|_{\lambda=1}Q_{k}$ becomes \emph{Monte Carlo} method. It is well-known that $\lambda$ trades off the bias of the bootstrapping with an approximate $q^{\pi}$, with the variance of sampling multi-step returns estimation ~\cite{kearns2000bias}. In practice, a high and intermediate $\lambda$ should be typically better ~\cite{singh1998analytical,sutton1996generalization}.

\section{Mixed-sampling Operator}

In this section, we present the \emph{mixed-sampling operator} $\mathcal{T}^{\pi,\mu}_{\sigma}$, which is one of our key contribution and is flexible to analysis our new algorithm later.
By introducing a sampling parameter $\sigma\in[0,1]$, the mixed-sampling operator varies continuously from pure-expectation method to full-sampling method.
In this section, we analysis the contraction of $\mathcal{T}^{\pi,\mu}_{\sigma}$ firstly. Then we introduce the $\lambda$-return vision of mixed-sampling operator, denoting it $\mathcal{T}^{\pi,\mu}_{\sigma,\lambda}$. Finally, we give a upper error bound of the corresponding policy evaluation algorithm.
\subsection{Contraction of Mixed-sampling Operator}
\begin{myDef}
Mixed sampling operator $\mathcal{T}^{\pi,\mu}_{\sigma}$ is a map on $\mathbb{R}^{|\mathcal{S}|\times|\mathcal{A}|}$ to $\mathbb{R}^{|\mathcal{S}|\times|\mathcal{A}|},$ $\forall s\in\mathcal{S}, a\in\mathcal{A}, \sigma\in[0,1]:$
\begin{flalign}
\nonumber
\mathcal{T}^{\pi,\mu}_{\sigma}: \mathbb{R}^{|\mathcal{S}|\times|\mathcal{A}|}&\rightarrow \mathbb{R}^{|\mathcal{S}|\times|\mathcal{A}|}\\
q(s,a)&\mapsto q(s,a)+\mathbb{E}_{\mu}\sum_{t=0}^{\infty}\Big[\gamma^{t}\delta_{t}^{\pi,\sigma}\Big] 
\end{flalign}
where
\begin{eqnarray*}
\delta_{t}^{\pi,\sigma}&=&\sigma\Big(R_{t}+\gamma Q(S_{t+1},A_{t+1})-Q(S_{t},A_{t})\Big)\\
&&\hspace{0.1cm}+(1-\sigma)\Big(R_{t}+\gamma\mathbb{E}_{\pi}Q(S_{t+1},\cdot)-Q(S_{t},A_{t})\Big)\\ 
\end{eqnarray*}
$$\mathbb{E}_{\pi}Q(S_{t+1},\cdot)=\sum_{a\in\mathcal{A}}\pi(S_{t+1},a)Q(S_{t+1},a)$$
\end{myDef}

The parameter $\sigma$ is also \emph{degree} of sampling intrduced by the $Q(\sigma)$ algorithm ~\cite{de2018multi}. 
In one of extreme end ($\sigma=0$, pure-expectation), $\mathcal{T}^{\pi,\mu}_{\sigma=0}$ can deduce the $n$-step returns $G_{n}^{\pi}$ in $Q^{\pi}(\lambda)$~\cite{H2016}, where
$G_{n}^{\pi}=\sum_{k=t}^{t+n}\gamma^{k-t}\delta^{ES}_{k}+\gamma^{n+1}\mathbb{E}_{\pi}Q(S_{t+n+1},\cdot)$, $\delta_{k}^{ES}$ is the \emph{k}-th expected TD error.
\emph{Multi-step Sarsa} ~\cite{sutton2017reinforcement} is in another  extreme end ($\sigma=1$, full-sampling).
Every intermediate value $\sigma$ can create a mixed method varies continuously from pure-expectation to full-sampling which is why we call $\mathcal{T}^{\pi,\mu}_{\sigma}$ mixed sample operator.

$\lambda$-\textbf{Return Version} We now define the $\lambda$-version of $\mathcal{T}^{\pi,\mu}_{\sigma}$, denote it as $\mathcal{T}^{\pi,\mu}_{\sigma,\lambda}$:
\begin{flalign}
\mathcal{T}^{\pi,\mu}_{\sigma,\lambda}q(s,a)=q(s,a)+\mathbb{E}_{\mu}[\sum_{t=0}^{\infty}(\lambda\gamma)^{t}\delta_{t}^{\pi,\sigma}] 
\end{flalign}
where the $\lambda$ is the parameter takes the from TD(0) to Monte  Carlo version as usual. When $\sigma=0$, $\mathcal{T}^{\pi,\mu}_{\sigma=0,\lambda}$ is restricted to $\mathcal{T}^{\pi,\mu}_{\lambda}$ ~\cite{H2016}, when $\sigma=1$, $\mathcal{T}^{\pi,\mu}_{\sigma=1,\lambda}$ is restricted to $\lambda$-operator.
The next theorem provides a basic property of $\mathcal{T}^{\pi,\mu}_{\sigma,\lambda}$.

\begin{theorem}
The operator $\mathcal{T}^{\pi,\mu}_{\sigma,\lambda}$ is a $\gamma$-contraction: for any $Q_{1},Q_{2}$,
$$\|\mathcal{T}^{\pi,\mu}_{\sigma,\lambda}q_{1}-\mathcal{T}^{\pi,\mu}_{\sigma,\lambda}q_{2}\|\leq\gamma\|q_{1}-q_{2}\|$$
Furthermore, for any initial $Q_{0}$, the sequence $\{Q\}_{k=0}^{\infty}$ is generated by the iteration
$$Q_{k+1}=\mathcal{T}^{\pi,\mu}_{\sigma,\lambda}Q_{k}$$
can converge to  the unique fixed point of $\mathcal{T}^{\pi,\mu}_{\sigma,\lambda}$.
\end{theorem}
\begin{proof}
Unfolding the operator $\mathcal{T}^{\pi,\mu}_{\sigma,\lambda}$, we have
\begin{flalign}
\nonumber
&\mathcal{T}^{\pi,\mu}_{\sigma,\lambda}q\\
\nonumber
&=\sigma(q+\mathbb{E}_{\mu}[\sum_{t=0}^{\infty}(\gamma\lambda)^{t}\delta_{t}^{S}])+(1-\sigma)(q+\mathbb{E}_{\mu}[\sum_{t=0}^{\infty}(\gamma\lambda)^{t}\delta_{t}^{ES}])\\
&=\sigma\underbrace{\Big(q+B[\mathcal{T}^{\mu}q-q]\Big)}_{\mathcal{T}^{\mu}_{\lambda}q}+(1-\sigma)\underbrace{\Big(q+B[\mathcal{T}^{\pi}q-q]\Big)}_{\mathcal{T}^{\pi,\mu}_{\lambda}q}
\end{flalign}
where $B=(I-\gamma\lambda\mathcal{P}^{\mu})^{-1}$. Based the fact that both $\mathcal{T}^{\mu}_{\lambda}$ \cite{bertsekas2012dynamic}and $\mathcal{T}^{\pi,\mu}_{\lambda}$~\cite{H2016,munos2016safe} are $\gamma$-contraction  operators, and  $\mathcal{T}^{\pi,\mu}_{\sigma,\lambda}$ is the convex combination of above operators, thus $\mathcal{T}^{\pi,\mu}_{\sigma,\lambda}$ is a $\gamma$-contraction.
\end{proof}

\subsection{Upper Error Bound of Policy Evaluation}

In this section we discuss the ability of policy evaluation iteration $Q_{k+1} = \mathcal{T}^{\pi,\mu}_{\sigma,\lambda}Q_{k}$ in Theorem 1. Our results show that when $\mu$ and $\pi$ are sufficiently close, the ability of the policy evaluation iteration increases gradually as the $\sigma$ decreases from 1 to 0.
\begin{lemma}
If a sequence $ \{a_{k}\}_{k=1}^{\infty}$ satisfies $a_{k+1}\leq\alpha a_{k}+\beta$, then for any $|\alpha|<1$, we have
$$a_{k}-\frac{\beta}{1-\alpha}\leq \alpha^{k}(a_{1}-\frac{\beta}{1-\alpha})$$
Furthermore, for any $\epsilon>0$, $\exists K, s.t, \forall k > K$ has the following estimation
$$|a_{k}|\leq\frac{\beta}{\alpha-1}+\epsilon.$$
\end{lemma}

\begin{theorem}[Upper error bound of policy evaluation]
Consider the policy evaluation algorithm $Q_{k+1} = \mathcal{T}^{\pi,\mu}_{\sigma,\lambda}Q_{k}$, if the behavior policy $\mu$ is $\epsilon$-away from
the target policy $\pi$, in the sense that $\max_{s\in\mathcal{S}}\|\pi(s,a)-\mu(s,a)\|_{1}\leq\epsilon$, $\epsilon<\frac{1-\gamma}{\lambda\gamma}$, and $\gamma(1+2\lambda)<1$,
then for a large $k$, the policy evaluation sequence$\{Q_{k}\}$ satisfy
$$\|Q_{k+1}-q^{\pi}\|_{\infty}\leq\sigma\epsilon\Big[\frac{M+\gamma C}{\gamma(1+2\lambda)-1}+1\Big]$$
where for a given policy $\pi$, $M,C$ is determined by the learning system.
\end{theorem}

\begin{proof}
Firstly, we provide an equation which could be used later:
\begin{flalign}
\nonumber
&\hspace{0.7cm}q+B[\mathcal{T}^{\mu}q-q]-q^{\pi}\\
\nonumber
&=B[(I-\gamma\lambda\mathcal{P}^{\mu})(q-q^{\pi})+\mathcal{T}^{\mu}q-q]\\
\nonumber
&=B[-\mathcal{T}^{\pi}q^{\pi}+\mathcal{T}^{\mu}q^{\pi}-\mathcal{T}^{\mu}q^{\pi}+\mathcal{T}^{\mu}q+\gamma\lambda\mathcal{P}^{\mu}(q^{\pi}-q)]\\
\nonumber
&=B\Big[\underbrace{\mathcal{R}^{\mu}-\mathcal{R}^{\pi}+\gamma(\mathcal{P}^{\mu}-\mathcal{P}^{\pi})q^{\pi}}_{-\mathcal{T}^{\pi}q^{\pi}+\mathcal{T}^{\mu}q^{\pi}}+\underbrace{\gamma\mathcal{P}^{\mu}(q-q^{\pi})}_{-\mathcal{T}^{\mu}q^{\pi}+\mathcal{T}^{\mu}q}\\
&+\gamma\lambda\mathcal{P}^{\mu}(q^{\pi}-q)\Big].
\end{flalign}
Rewrite the policy evaluation iteration:
$$Q_{k+1}=\sigma\mathcal{T}_{\lambda}^{\mu}Q_{k}+(1-\sigma)\mathcal{T}_{\lambda}^{\pi,\mu}Q_{k}.$$
Note $q^{\pi}$ is fixed point of $\mathcal{T}_{\lambda}^{\pi,\mu}$~\cite{H2016}, then we merely consider next estimator:
\begin{eqnarray*}
&&\|Q_{k+1}-q^{\pi}\|_{\infty}\\
&=&\sigma\|Q_{k}+B[\mathcal{T}^{\mu}Q_{k}-Q_{k}]-q^{\pi}\|_{\infty}\\
&\leq&\sigma\Big[\frac{\epsilon(M+\gamma \|q^{\pi}\|)}{1-\gamma\lambda}+\frac{\gamma(1+\lambda)}{1-\gamma\lambda}\|Q_{k}-q^{\pi}\|_{\infty}\Big] \\
&\overset{\text{Lemma1}}{\leq}&\sigma\epsilon\Big[\frac{M+\gamma C}{\gamma(1+2\lambda)-1}+1\Big].
\end{eqnarray*}
The first equation is derived by replacing $q$ in (10) with $Q_{k}$.
Since $\mu$ is $\epsilon$-away from $\pi$, the first inequality is determined the following fact:
\begin{eqnarray*}
\|\mathcal{R}_{s}^{\pi}-\mathcal{R}_{s}^{\mu}\|_{\infty}&=&\max_{a\in\mathcal{A}}\{|(\pi(s,a)-\mu(s,a))\mathcal{R}_{s}^{a}|\}\\
&\leq&\epsilon|\mathcal{A}|\max_{a}|\mathcal{R}_{s}^{a}| =\epsilon M_{s},\\
\|\mathcal{R}^{\pi}-\mathcal{R}^{\mu}\|_{\infty}&\leq&\epsilon M,
\end{eqnarray*}
where $M_{s}=|\mathcal{A}|\max_{a}|\mathcal{R}_{s}^{a}|$ is determined by the reinforcement  learning system and independent of $\pi,\mu$. $M=\max_{s\in\mathcal{S}}M_{s}$.
For the given policy $\pi$, $\|q^{\pi}\|$ is a constant on determined by learning system, we denote it $C$.
\end{proof}

 \begin{remark}
 The proof in Theorem 2 strictly dependent on the assumption that $\epsilon$ is smaller but never to be zero, where the $\epsilon$ is a  bound of discrepancy between the behavior policy $\mu$ and target policy $\pi$. That is to say, the ability of the prediction in policy evaluation iteration is dependent on the gap between $\mu$ and $\pi$. 
 \end{remark}

\section{$Q(\sigma,\lambda)$ Control Algorithm}
In this section, we present $Q(\sigma,\lambda)$ algorithm for control. 
We analysis the off-line version of $Q(\sigma,\lambda)$ which converges to optimal value function exponentially.

Considering the typical iteration $(Q_{k},\pi_{k})$, $\mu_{k}$ is an arbitrary sequence of corresponding behavior policies, $\pi_{k+1}$is calculated by the following two steps,
\\ \emph{Step}1: \emph{policy evaluation}
$$Q_{k+1}=\mathcal{T}^{\pi_{k},\mu_{k}}_{\sigma,\lambda}Q_{k}$$
\\ \emph{Step}2: \emph{policy improvement}
$$\mathcal{T}^{\pi_{k+1}}Q_{k+1}=\mathcal{T}^{*}Q_{k+1}$$
that is $\pi_{k+1}$ is greedy policy with repect to $Q_{k+1}$.
We call the approach introduced by above step1 and step2 $Q(\sigma,\lambda)$ \emph{control algorithm}.

In the following, we presents the convergence rate of $Q(\sigma,\lambda)$ control algorithm.
\begin{theorem}[Convergence of $Q(\sigma,\lambda)$ Control Algorithm]
Considering the sequence $\{(Q_{k},\pi_{k})\}_{k\ge 0}$ generated by the $Q(\sigma,\lambda)$ control algorithm, given $\lambda,\gamma\in(0,1)$, then
$$\|Q_{k+1}-q^{*}\|\leq\frac{\gamma(1+\lambda-2\lambda\sigma)}{1-\lambda\gamma}\|Q_{k}-q^{*}\|.$$
Particularly, for $\lambda<\frac{1-\gamma}{2\gamma}$, then sequence $\{Q_{k}\}_{k\ge 1}$ converges to $q^{*}$ exponentially fast:
$$\|Q_{k+1}-q^{*}\|=O\Big(\frac{\gamma(1+\lambda-2\lambda\sigma)}{1-\lambda\gamma}\Big)^{k+1}$$
\end{theorem}
\begin{proof} By the definition of $\mathcal{T}^{\pi,\mu}_{\sigma,\lambda}$,
$$
\mathcal{T}^{\pi,\mu}_{\sigma,\lambda}q=\sigma\mathcal{T}^{\mu}_{\lambda}q+(1-\sigma)\mathcal{T}^{\pi,\mu}_{\lambda}q
$$
we have\footnote{The section inequality is based on the next two results:~\cite{munos2016safe} Theorem2 and ~\cite{bertsekas2012dynamic} Proposition6.3.10.}:
\begin{eqnarray*}
&&\|Q_{k+1}-q^{*}\|\\
&\leq&\sigma\|\mathcal{T}^{\mu_{k}}_{\lambda}(Q_{k}-q^{*})\|+(1-\sigma)\|\mathcal{T}^{\pi_{k},\mu_{k}}_{\lambda}(Q_{k}-q^{*})\|\\
&\leq&\Big(\sigma\frac{\gamma(1-\lambda)}{1-\lambda\gamma}+(1-\sigma)\frac{\gamma(1+\lambda)}{1-\lambda\gamma}\Big)\|Q_{k}-q^{*}\|\\
&=&\frac{\gamma(1+\lambda-2\lambda\sigma)}{1-\lambda\gamma}\|Q_{k}-q^{*}\|
\end{eqnarray*}
\end{proof}

\section{On-line Implementation of $Q(\sigma,\lambda)$}
We have discussed the contraction of mixed-sampling operator $\mathcal{T}^{\pi,\mu}_{\sigma,\lambda}$ through which we introduced the $Q(\sigma,\lambda)$ control algorithm. Both of the iteration in Theorem 2 and Theorem 3 are the version of offline. In this section, we give the on-line version of $Q(\sigma,\lambda)$ and discuss its convergence.

\subsection{On-line Learning}
Off-line learning is too expensive due to the learning process must be carried out at the end of a episode, however, on-line learning updates value function with a lower computational cost,  better performance. There is a simple interpretation of equivalence between off-line learning and on-line learning which means that, by the end of the episode, the total updates of the \emph{forward view}(\emph{off-line learning}) is equal to the total updates of the \emph{backward view}(\emph{on-line learning})~\cite{sutton1998reinforcement}. By the view of equivalence\footnote{The true online learning was firstly introduced by~\cite{seijen2014true}, more details in \cite{van2016true}.}, on-line learning can be seen as an implementation of offline algorithm in an inexpensive manner. Another interpretation of online learning was provided by~\cite{singh1996reinforcement}, TD learning with accumulate trace comes to approximate \emph{every-visit Monte-Carlo} method and TD learning with replace trace comes to approximate \emph{first-visit Monte-Carlo} method.

The iterations in Theorem 2 and Theorem 3 are the version of expectations . In practice, we can only access to the trajectory $\{(S_{t},A_{t},R_{t})\}_{t\ge0}$. By statistical approaches, we can utilize the trajectory to estimate the value function. Algorithm 1 corresponds to online form of $Q(\sigma,\lambda)$.
\begin{tabular}{lccc}
\hline
\textbf {Algorithm1:}On-line Q($\sigma,\lambda$) algorithm\\
\hline
\textbf {Require}:Initialize $Q_{0}(s,a)$ arbitrarily,\quad$\forall s\in\mathcal{S},\forall a\in\mathcal{A}$
\\\textbf {Require}:Initialize $\mu_{k}$ to be the behavior policy
\\ \textbf {Parameters}: step-size $\alpha_{t}\in(0,1]$\\
\textbf{Repeat} (for each episode):\\
\quad $Z(s,a)=0\hspace{0.1cm}\forall s\in\mathcal{S},  \forall a\in\mathcal{A}$\\
\quad$Q_{k+1}(s,a)=Q_{k}(s,a)\hspace{0.1cm}\forall s\in\mathcal{S},  \forall a\in\mathcal{A}$\\
\quad Initialize state-action pair $(S_{0},A_{0})$\\
\quad \textbf{For} $t$ = 0 , 1, 2, $\cdots$ $T_{k}$:\\
\quad\quad  Obersive a sample $(R_{t},S_{t+1},A_{t+1}) \sim\mu_{k}$\\
\quad\quad $\delta_{t}^{\sigma,\pi_{k}}=R_{t}+\gamma\Big\{ (1-\sigma)\mathbb{E}_{\pi_{k}}[Q_{k+1}(S_{t+1},\cdot)]$ \\
\quad\quad\quad\quad\quad$+\sigma Q_{k+1}(S_{t+1},A_{t+1})\Big\}-Q_{k+1}(S_{t},A_{t})$\\
\quad\quad \textbf{For} $\forall s\in\mathcal{S}, \forall a\in\mathcal{A}$:\\
\quad\quad\quad$Z(s,a)=\gamma\lambda Z(s,a)$+$\mathbb{I}\{(S_{t},A_{t})=(s,a)\}$\\
\quad\quad\quad$Q_{k+1}(s,a) = Q_{k+1}(s,a) + \alpha_{k}\delta_{t}^{\sigma,\pi_{k}}Z(s,a)$\\
\quad\quad\textbf{End  For}\\
\quad\quad$S_{t+1} =S_{t}$,  $A_{t+1} =A_{t}$ \\
\quad\quad\textbf{If} $S_{t+1}$ is terminal:\\
\quad\quad\quad\textbf{Break}\\
\quad \textbf{End For}\\
\hline
\end{tabular}

\subsection{On-line Learning Convergence Analysis}
We make some common assumption similar to ~\cite{bertsekas1996book,H2016}.
\begin{assumption}
$\sum_{t\ge0}P[(S_{t},A_{t})=(s,a)]>0$, minimum visit frequency, every pair $(s,a)$ can be visited.
\end{assumption}
\begin{assumption}
For every historical  chain $\mathcal{F}_{t}$ in a MDPs, $P[N_{t}(s,a)\ge k|\mathcal{F}_{t}]\leq \gamma\rho^{k}$, where $\rho$ is a positive constants, $k$ is a positive integer.
\end{assumption}
For the convenience of expression, we give some notations firstly.
Let $Q_{k,t}^{o}$ be the vector obtained after $t$ iterations in the $k$-th trajectory, and the superscript $o$ emphasizes online learning. We denote the $k$-th trajectory as $\{(S_{t},A_{t},R_{t})\}_{t\ge0}$ sampled by the policy $\mu_{k}$. Then the online update rules can be expressed as follows: $\forall\hspace{0.1cm} (s,a)\in\mathcal{S}\times\mathcal{A}$
\begin{eqnarray*}
Q_{k,0}^{o}(s,a)&=&Q_{k}^{o}(s,a)\\
\delta_{k,t}^{o}&=&R_{t}+\gamma Q_{k,t}^{o}(S_{t+1},A_{t+1})-Q_{k,t}^{o}(S_{t},A_{t})\\
Q_{k,t+1}^{o}(s,a)&=&Q_{k,t}^{o}(s,a)+\alpha_{k}(s,a)z_{k,t}^{o}(s,a)\delta_{k,t}^{o}\\
Q_{k+1}^{o}(s,a)&=&Q_{k,T_{k}}^{o}(s,a)
\end{eqnarray*}
where $T_{k}$ is the length of the $k$-$th$ trajectory.
\begin{theorem}
Based on the Assumption 1 and Assumption 2, step-size $\alpha_{t}$ satisfying,$ \sum_{t} \alpha_{t}=\infty,\sum_{t} \alpha^{2}_{t}<\infty$, $\pi_{k}$ is greedy with respect to $Q^{o}_{k}$, then $Q^{o}_{k}\xrightarrow{w.p.1} q^{*}$, where $w.p.1$ is short for with probability one.
\end{theorem}
\begin{proof}

After some sample algebra:
$$Q_{k+1}^{o}(s,a)=Q_{k}^{o}(s,a)+\widetilde{\alpha_{k}}\Big(G_{k}^{o}(s,a)-\frac{N_{k}(s,a)}{\mathbb{E}[N_{k}(s,a)]}Q_{k}^{o}(s,a)\Big),$$
$$G_{k}^{o}(s,a)=\frac{1}{\mathbb{E}[N_{k}(s,a)]}\sum_{t=0}^{T_{k}}z_{k,t}^{o}(s,a)\delta_{k,t}^{o},$$
where $\widetilde{\alpha_{k}}=\mathbb{E}[N_{k}(s,a)]\alpha_{k}(s,a)$. We rewrite the off-line update:
$$Q_{k+1}^{f}(s,a)=Q_{k}^{f}(s,a)+\alpha_{k}\Big(G_{k}^{f}(s,a)-\frac{N_{k}(s,a)}{\mathbb{E}[N_{k}(s,a)]}Q_{k}^{f}(s,a)\Big),$$
$$G_{k}^{f}=\frac{1}{\mathbb{E}[N_{k}(s,a)]}\sum_{t=0}^{N_{k}(s,a)}Q_{k,t}^{\lambda}(s,a),$$
where $Q_{k,t}^{\lambda}(s,a)$ is the $\lambda$-returns at time $t$ when the pair $(s,a)$ was visited in the $k$-th trajectory,
 the superscript $f$ in $Q_{k+1}^{f}(s,a)$ emphasizes the forward (off-line) update. $N_{k}(s,a)$ denotes the times of the pair $(s,a)$ visited in the $k$-th trajectory.
\\ We define the \emph{residual}
between $G^{o}_{k}$ and the off-line estimate $G^{f}_{k}(s,a)$ in the $k$-th trajectory:
 $$Res_{k}(s,a)=Q^{o}_{k}(s,a)-Q^{f}_{k}(s,a).$$
Set $\Delta_{k}(s,a)=Q_{k}^{o}(s,a)-q^{*}(s,a)$, then we consider the next random iterative process:
\begin{flalign}
\Delta_{k+1}(s,a)=(1-\hat{\alpha_{k}}(s,a))\Delta_{k}(s,a)+\hat{\beta}_{k}F_{k}(s,a),
\end{flalign}
where
$$\hat{\alpha_{k}}(s,a)=\frac{N_{k}(s,a)\alpha_{k}(s,a)}{\mathbb{E}_{\mu_{k}}[N_{k}(s,a)]},\hat{\beta}_{k}(s,a)=\alpha_{k}(s,a),$$
$$F_{k}(s,a)=G_{k}^{o}(s,a)-\frac{N_{k}(s,a)}{\mathbb{E}_{\mu_{k}}[N_{k}(s,a)]}q^{*}(s,a).$$
\textbf{Step1:Upper bound on} $Res_{k}(s,a)$:
\begin{flalign}
\max_{(s,a)}\Big|\mathbb{E}_{\mu_{k}}[Res_{k}(s,a)]\Big|\leq C_{k}\max_{(s,a)}\Big|Q_{k+1}^{o}(s,a)-q^{*}\Big|,
\end{flalign}
where $C_{k}\xrightarrow{w.p.1}0$.
$$Res_{k,t}(s,a)=\frac{1}{\mathbb{E}[N_{k}(s,a)]}\sum_{m=0}^{t}\Big[z_{k,m}^{o}(s,a)\delta_{k,m}^{o}-Q_{k,m}^{\lambda}(s,a)\Big],$$
where $0\leq t\leq T_{k}.$
$Res_{k,t}(s,a)$ is the difference between the total on-line updates of first $t$ steps and the first $t$ times off-line update in $k$-th trajectory. By induction on $t$, we have:
$$\|Res_{k,t+1}(s,a)\|\leq \alpha_{M}C(\Delta+\|Res_{k,t}(s,a)\|),$$
where $C$ is a consist and $\Delta=\|Q(s,a)-q^{*}(s,a)\|$, $\alpha_{M} = \max_{0\leq t\leq T_{k}}\{\alpha_{t}(s,a)\}$.
Based on the condition of step-size in the Theorem 4, $\alpha_{M}\xrightarrow{w.p.1}0$, then we have (12).\\
\textbf{Step2:} $\max_{(s,a)}\mathbb{E}_{\mu_{k}}[F_{k}(s,a)]\leq\gamma\max_{(s,a)}\|\Delta_{k}(s,a)\|$.
\\In fact:
$$F_{k}(s,a)=G_{k}^{f}(s,a)+Res_{k}(s,a)-\frac{N_{k}(s,a)}{\mathbb{E}_{\mu_{k}}[N_{k}(s,a)]}q^{*}(s,a)$$
\begin{eqnarray*}
\mathbb{E}_{\mu_{k}}[F_{k}(s,a)]&=&\frac{\sum_{t=0}^{N_{k}(s,a)}\mathbb{E}_{\mu_{k}}[Q_{k,t}^{\lambda}(s,a)-q^{*}]}
{\mathbb{E}_{\mu_{k}}[N_{k}(s,a)]}+Res_{k}(s,a)
\end{eqnarray*}
From the property of eligibility trace(more details refer to~\cite{bertsekas2012dynamic}) and Assumption 2, we have:
\begin{eqnarray*}
\Big|\mathbb{E}_{\mu_{k}}[Q_{k,t}^{\lambda}(s,a)-q^{*}]\Big|&\leq&P[N_{k}(s,a)\ge t]\mathbb{E}_{\mu_{k}} \Big|Q_{k}^{\lambda}(s,a)-q^{*}\Big|,\\
&\leq&\gamma\rho^{t}\max_{(s,a)}|Q^{f}_{k}(s,a)-q^{*}|,
\end{eqnarray*}
Then according to (11), for some $ t >0$:
$$\mathbb{E}_{\mu_{k}}[F_{k}(s,a)]\leq(\gamma\rho^{t}+\alpha_{t})\max_{(s,a)}\|\Delta_{k}(s,a)\leq\gamma\max_{(s,a)}\|\Delta_{k}(s,a)\|.$$
\textbf{Step3}: $Q^{o}_{k}\xrightarrow{w.p.1} q^{*}$
Considering the iteration (11) and Theorem 1 in ~\cite{jaakkola1994convergence}, then we have $Q^{o}_{k}\xrightarrow{w.p.1} q^{*}$.
\end{proof}
Based on Theorem 3 in \cite{munos2016safe} and our Theorem 4, if $\pi_{k}$ is greedy with respect to $Q_{k}$, then $Q_{k}$ in Algorithm 1 can converge to $q^{*}$ with probability one.
\\ \textbf{Remark 2} \emph{The conclusion in~\cite{jaakkola1994convergence} similar to our Theorem 4, but the update is different from ours and we further develop it under the Assumption 2}.

\section{Experiments}

\subsection{Experiment for Prediction Capability}

In this section, we test the prediction abilities of $Q(\sigma ,\lambda)$ in \emph{19-state random walk} environment which is a one-dimension MDP environment that widely used in reinforcement learning~\cite{sutton2017reinforcement,de2018multi}. The agent at each state has two action : \emph{left} and \emph{right}, and taking each action with equal probability.

We compare the root-mean-square(RMS) error as a function of episodes, $\sigma$ varies dynamically $\sigma$ from 0 to 1 with steps of 0.2. Results in Figure 1 show that the performance of $Q(\sigma ,\lambda)$ increases gradually as the $\sigma$ decreases from 1 to 0, which just verifies the upper error bound in Theorem2. 
\begin{figure}[htbp]
  \centering
   \includegraphics[width=0.23\textwidth]{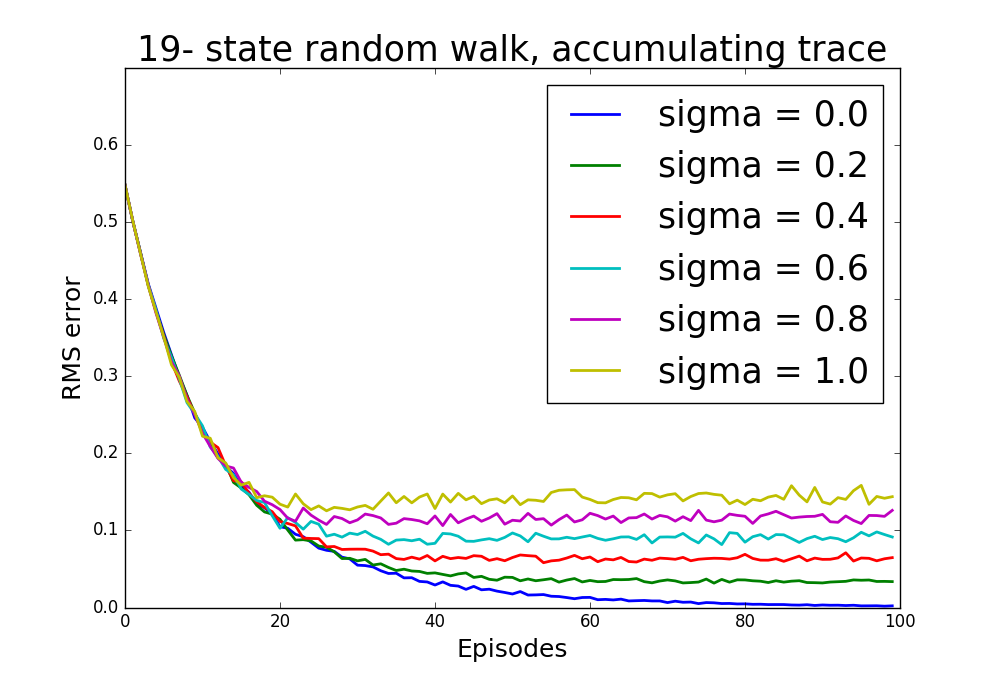}
  \includegraphics[width=0.23\textwidth]{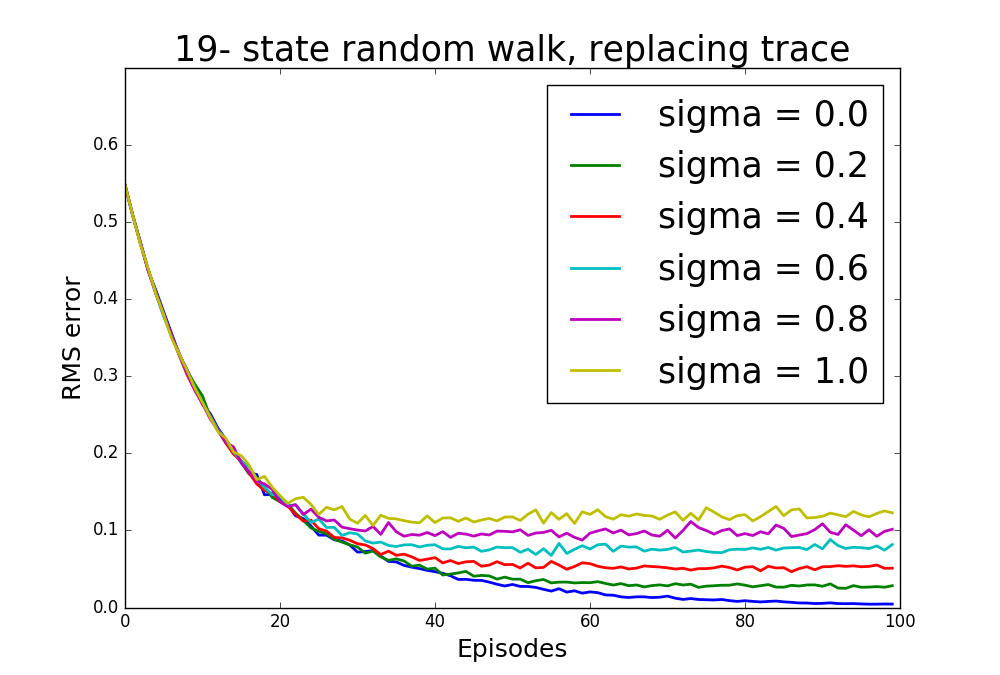}
  \caption{Root-mean-square(RMS) error of state values as a function of episodes in 19-state random walk, we consider both accumulating trace and replacing trace. The plot shows the prediction ability of $Q(\sigma,\lambda)$ dynamically varying $\sigma$, where $\lambda$ is fixed to 0.8,$\gamma=1$.}
\end{figure}

\subsection{Experiment for Control Capability}
We test the control capability of $Q(\sigma,\lambda)$ in the classical episodic task, \emph{mountain car} \cite{sutton1998reinforcement}. Because the state space in this environment is continuous, we use \emph{tile coding} function approximation \cite{sutton1998reinforcement}, and use the version 3 of Sutton's tile coding\footnote{http://incompleteideas.net/rlai.cs.ualberta.ca/RLAI/RLtoolkit/tilecoding.html} software (n.d.) with 8 tilings. 
\begin{figure}[htbp]
  \centering
  \includegraphics[width=0.23\textwidth]{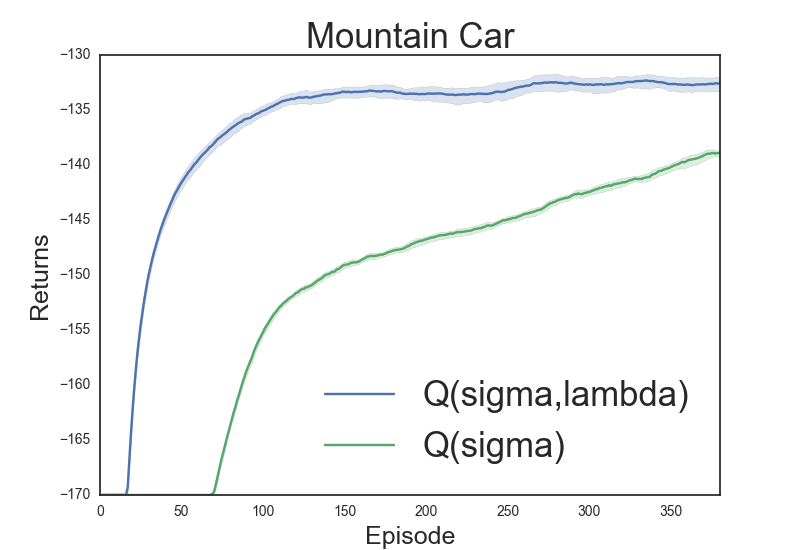}
  \includegraphics[width=0.23\textwidth]{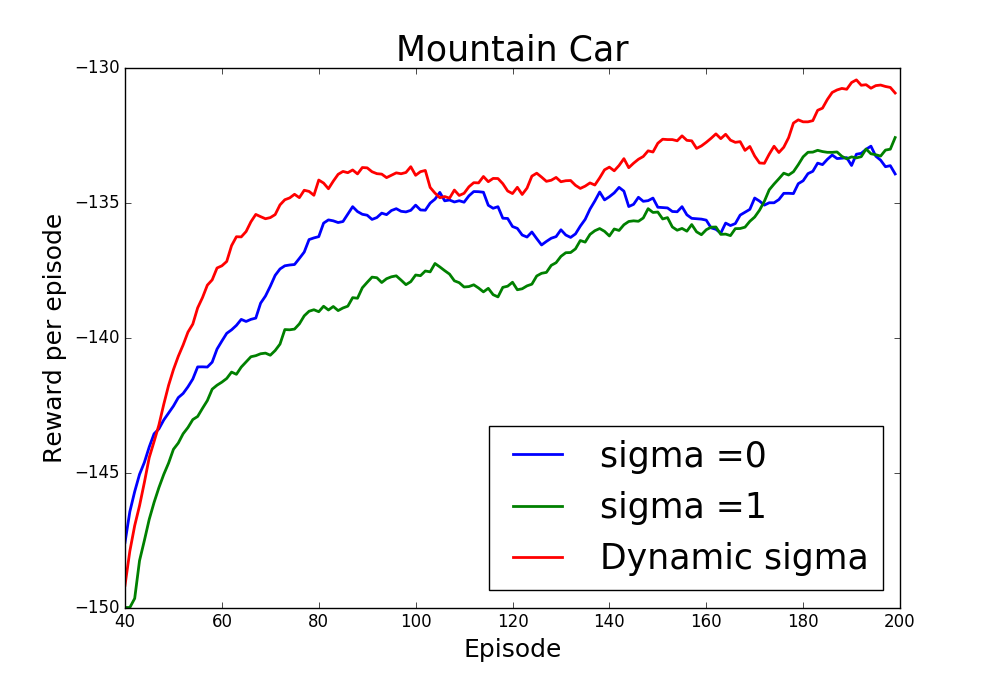}
  \caption{The plot shows the the average return per episode. A right-centered moving average with a window of 20 successive episodes was employed in order to smooth the results. 
  The right plot shows the result of $Q(\sigma,\lambda )$ comparing with $Q(\sigma)$. The left plot shows the result of $Q(\sigma,\lambda )$ comparing with $Q^{\pi}(\sigma)$ ($\sigma=0$) and Sarsa$(\lambda)$ ($\sigma=1$). $\gamma=0.99$, step-size $\alpha=0.3$.
}
\end{figure}

\begin{table}[htb]
 \centering
 \caption{Average Return per Episode After 50 Episodes. 
 }
 
 \label{tab:comparison}
 \begin{tabular}{l|c|c|c|c}
  \hline 
  Algorithm & $\text{Mean}$ & $\text{UB}$ & $\text{LB}$   \\ \hline
  $Q(\sigma=0,\lambda)$,$Q^{\pi}(\lambda)$ & \textbf{-193.56} & -179.33 & -197.06  \\ \hline
  $Q(\sigma=1,\lambda)$,$Sarsa(\lambda)$ & -196.84 & -182.69 & -200.42  \\ \hline
  $Q(\sigma=0.5,\lambda)$ & -195.99 & -181.60 & -199.62 \\ \hline
  Dynamic $\sigma$   & -195.01 & \textbf{-177.14} & \textbf{-195.01}  \\ \hline
   \end{tabular}
\end{table}

In the right part of Figure 2, we collect the data by varing $\sigma$ from $0$ to $1$ with steps of $0.02$. Results show that $Q(\sigma,\lambda)$ significantly converges faster than $Q(\sigma)$. 
In the left part of Figure 2, results show that  the $Q(\sigma,\lambda)$ with $\sigma$ in an intermediate value can outperform $Q^{\pi}(\lambda)$ and Sarsa$(\lambda)$.
Table1 and Table2 summarize the average return after 50 and 200 episodes. In order to gain more insight into the nature of the results, we run $\sigma$ from $0$ to $1$ with steps of $0.02$, we take the statistical method according to \cite{de2018multi}, lower (LB) and upper (UB) 95$\%$ confidence interval bounds are provided to validate the results. The average return after only 50 episodes could be interpreted as a measure of initial performance, whereas the average return after 200 episodes shows how well an algorithm is capable of learning\cite{de2018multi}. Results show that $Q(\sigma,\lambda)$ with a intermediate value  had the best final performance.

\begin{table}[htb]
 \centering
 \caption{Average Return per Episode After 200 Episodes. 
 }
 
 \label{tab:comparison}
 \begin{tabular}{l|c|c|c|c}
  \hline 
  Algorithm & $\text{Mean}$ & $\text{UB}$ & $\text{LB}$   \\ \hline
  $Q(\sigma=0,\lambda)$,$Q^{\pi}(\lambda)$ & -144.04 & -138.49 & -146.44  \\ \hline
  $Q(\sigma=1,\lambda)$,$Sarsa(\lambda)$ & -145.72 & -140.04 & -148.21  \\ \hline
  $Q(\sigma=0.5,\lambda)$ & -143.71 & -137.92 & -146.12 \\ \hline
  Dynamic $\sigma$   & \textbf{-142.62} & \textbf{-137.24} & \textbf{-145.09}  \\ \hline
   \end{tabular}
\end{table}
\section{Conclusion}
In this paper we presented a new method, called Q$(\sigma ,\lambda)$, which unifies Sarsa$(\lambda)$ and $Q^{\pi}(\lambda)$. We solved a upper error bound of $Q(\sigma,\lambda)$ for the ability of policy evaluation. Furthermore, we proved the convergence of Q$(\sigma ,\lambda)$ control algorithm to $q^{*}$ under some conditions. The proposed approach was compared with one-step and multi-step TD learning methods, results demonstrated its effectiveness.

\clearpage
\bibliographystyle{named}
\bibliography{ijcai18}

\end{document}